\documentclass{article}

     \PassOptionsToPackage{numbers, compress}{natbib}
\usepackage[preprint]{neurips_2025}

\usepackage[utf8]{inputenc} % allow utf-8 input
\usepackage[T1]{fontenc}    % use 8-bit T1 fonts
\usepackage{hyperref}       % hyperlinks
\usepackage{url}            % simple URL typesetting
\usepackage{booktabs}       % professional-quality tables
\usepackage{graphicx}
\usepackage{amsfonts}       % blackboard math symbols
\usepackage{nicefrac}       % compact symbols for 1/2, etc.
\usepackage{microtype}      % microtypography
\usepackage{xcolor}         % colors
\usepackage{amsmath, amsthm}
\usepackage{algorithm}
\usepackage{algpseudocode}
\usepackage{xspace}
\usepackage{multirow}
\usepackage{acronym}
\usepackage{makecell}
\usepackage{placeins}
\usepackage{arydshln}
\newtheorem{theorem}{Theorem}

\newtheorem{lemma}[theorem]{Lemma}

\theoremstyle{remark}
\newtheorem{remark}[theorem]{Remark}
\usepackage{wrapfig}
\usepackage{pgfplots}
\pgfplotsset{compat=1.18}

\title{K$^*$-Means: A Parameter-free Clustering Algorithm}
\newcommand{\method}{\texttt{k$^*$means}\xspace}

\DeclareMathOperator*{\argmin}{argmin}
\DeclareMathOperator{\Var}{Var}

\author{
  Louis Mahon\thanks{\url{https://lou1sm.github.io/louismahonville/}} \\
  School of Informatics\\
  University of Edinburgh\\
  \texttt{lmahonology@gmail.com} \\
  \And 
  Mirella Lapata \\
  School of Informatics\\
  University of Edinburgh\\  
}

\begin{document}

\acrodef{bic}[BIC]{Bayesian information criterion}
\acrodef{mdl}[MDL]{minimum description length}
\acrodef{gmm}[GMM]{Gaussian mixture model}

\maketitle

\begin{abstract}
  Clustering is a widely used and powerful machine learning technique, but its effectiveness is often limited by the need to specify the number of clusters,~$k$, or by relying on thresholds that implicitly determine~$k$. We introduce \method, a novel clustering algorithm that eliminates the need to set $k$~or any other parameters. Instead, it uses the minimum description length principle to automatically determine the optimal number of clusters, $k^*$, by splitting and merging clusters while also optimising the standard $k$-means objective. We prove that \method is guaranteed to converge and demonstrate experimentally that it significantly outperforms existing methods in scenarios where~$k$ is unknown. We also show that it is accurate in estimating~$k$, and that empirically its runtime is competitive with existing methods, and scales well with dataset size. %Code will be made available on publication.
\end{abstract}

\section{Introduction}
Clustering is a fundamental task in machine learning. As well as allowing data visualisation and exploration, it is used for several more specific functions in the context of machine learning systems, such as representation learning \cite{liu2023dinosr, niu2024owmatch}, federated learning \cite{ma2023cam}, exploration in reinforcement learning \cite{wagner2024just}, anomaly detection \cite{Markovitz_2020_CVPR}, and has found widespread application in the natural sciences \cite{xu2025analyzing, kisi2025integration, meyer2025uncovering, hebdon2025dog}. It has also been interwoven with deep learning feature extraction in the areas of deep clustering \cite{caron2018deep, mahon2021selective, miklautz2024breaking, liu2023deep, vo2024automatic} and deep graph clustering \cite{mo2024revisiting, fini2023semi}.
Clustering can produce meaningful and interpretable partitions of data, even in the absence of information typically required by other machine learning methods, such as annotated class labels.

However, almost all existing clustering algorithms still require some user-set parameters, which limits their applicability to cases where the user can choose appropriate values. Two common classes of clustering algorithms are centroid-based and density-based. The former, typified by $k$-means, work by finding the optimal location for cluster centre-points (centroids), and then assigning points to nearby centres. These algorithms generally require the user to specify the number of clusters.  Density-based algorithms aim to locate clusters where the density of points is high. They also require some threshold(s) to determine what constitutes a high-density region and where to separate them. 

In this paper, we design a clustering algorithm that eliminates the need for setting the number of clusters, tunable thresholds, or any other parameters. Our algorithm, \method, extends $k$-means by automatically determining the optimal number of clusters, $k$, using the \ac{mdl} principle. The MDL principle states that the optimum data representation is that containing the fewest bits. It has been shown to be effective in a number of applied tasks, including complexity quantification \cite{mahon2024minimum, mahon2024towards, mahon2025local} and temporal segmentation \cite{mahon-lapata-2024-modular, mahon2025parameter}. \method uses MDL by optimising the information-theoretic objective of minimising the description length of the data under the model. Specifically, this is quantified as the number of bits required to represent the cluster centroids, and the cluster labels of each point, which we refer to as the \emph{index cost}, along with the number of bits needed to represent the displacement of each point from its centroid, termed the \emph{residual cost}. Too many clusters creates a prohibitively high index cost, while too few creates a prohibitively high residual cost, so the objective guides the model towards a reasonable value of $k$. 
We optimise this objective by including in the model two subclusters of every cluster. The ``assign'' and ``update'' steps of $k$-means are applied to the subclusters in the same way as to the main clusters, and the algorithm has the option to split a cluster into its two subclusters, or to merge two clusters together, if it will reduce the overall description length. 

Despite its simplicity, $k$-means remains the most widely used clustering algorithm, because it is fast, provably guaranteed to converge, has just one easily interpretable parameter, and achieves accuracy competitive with more complicated methods. We aim to maintain these advantages with \method. We provide a proof that \method is also guaranteed to converge. Additionally, our experiments show that \method largely maintains the speed and accuracy advantages of $k$-means. It is as fast or faster than most other $k$-agnostic clustering methods, scales well with dataset size, and is close to or on par with the accuracy of $k$-means, even when $k$-means has an oracle for the true value of $k$. We also show in synthetic experiments that it can successfully identify $k$ more accurately than existing methods. Our contributions are summarised as follows:
\begin{itemize}
    \item We introduce \method, an entirely parameter-free clustering algorithm;
    \item We give a formal proof that \method will convergence in finite time;
    \item We design synthetic data experiments to test whether \method can infer the true value of $k$, and show that it can with much higher accuracy than existing methods;
    \item We show experimentally that, with respect to standard clustering metrics, it is more accurate than all existing methods that do not require setting $k$, and is as fast as or faster than the majority of these methods. 
\end{itemize}
The remainder of this paper is organised as follows. Section~\ref{sec:related-work} discusses related work, Section~\ref{sec:method} describes the algorithm of \method in detail, Section \ref{sec:experimental-eval} presents experimental results, and finally Section \ref{sec:conclusion} concludes and summarises.

\section{Related Work} \label{sec:related-work}
Two well-known centroid-based clustering algorithsm are $k$-means,  \cite{macqueen1967classification, lloyd1982least} and \acp{gmm}  \cite{dempster1977maximum}. The former partition data into $k$~clusters by iteratively assigning points to the nearest centroid and updating centroids until convergence, and the latter which fit a multivariate normal model via expectation maximization. A number of more complex clustering algorithms are also in widespread use.  

Spectral Clustering \cite{ng2001spectral} transforms data using eigenvectors of a similarity matrix before applying a clustering algorithm such as $k$-means. 
%Hierarchical clustering \cite{johnson1967hierarchical} builds a tree of clusters by recursively merging or splitting them based on distance between points. 
Mean Shift \cite{comaniciu2002mean} discovers clusters by iteratively shifting points toward areas of higher density until convergence. It does not require setting $k$, but does require a bandwidth parameter. Affinity Propagation \cite{frey2007clustering} identifies exemplars among data points and forms clusters by exchanging messages between pairs of samples until convergence. Like mean shift, it does not require specifying the number of clusters~$k$, but instead relies on a preference parameter and a damping factor. A common drawback of both mean shift and affinity propagation is their quadratic space complexity, which limits scalability.

DBSCAN \cite{ester1996density} identifies dense regions as clusters by grouping points with many neighbours, while marking sparse points as noise. OPTICS \cite{ankerst1999optics} extends DBSCAN by ordering points based on reachability distance, allowing it to identify clusters with varying densities. HDBSCAN \cite{campello2013density} further builds on DBSCAN by introducing a hierarchical clustering framework that extracts flat clusters based on stability. Although DBSCAN and its variants do not require specifying the number of clusters, they rely on other parameters—such as \texttt{eps} and \texttt{min-pts}, which specify the neighbourhood size and the number of points required to form a `dense region'. OPTICS avoids setting 
\texttt{eps}
by computing reachability distances over a range of values, but in its place introduces a steepness parameter to define cluster boundaries (where the reachability value decreases faster than this steepness). Tuning these parameters can yield
a wide range of values for the number of predicted clusters (see Appendix \ref{app:dbscan-vary-k}). Thus, in the absence of prior knowledge about the number of clusters or appropriate parameter values, DBSCAN and its derivatives can be difficult to apply effectively.

X-Means \cite{pelleg2000x} extends $k$-means by automatically determining the optimal number of clusters using the \ac{bic} \cite{schwarz1978estimating}. 
Our method is similar to X-Means in two respects: firstly, in that it selects~$k$ using an agnostic criterion from probability/information theory and secondly in that it considers bifurcating each centroid as the means by which to explore different values of~$k$. However, there are some important differences between the two methods. Our method uses \ac{mdl} as the criterion, whereas X-Means uses \ac{bic}. 
%We explore the effects of different criteria in Section TODO. 
Secondly, our method does not require the $max_K$ parameter. It can, in principle, return any value of~$k$ (although this would have to be bounded by~$N$). Thirdly, X-Means operates in two steps, returning a set of possible models by iteratively using local \ac{bic} on each cluster to determine whether it should split, and then using global \ac{bic} to select the best model from this set. This means it needs to run $k$-means to convergence multiple times, once for each model. \method, in contrast, returns the best model in one stage, by only splitting when it reduces the \ac{mdl}, and keeping a pre-initialised pair of sub-centroids for each cluster, which are updated one step at a time while $k$ is being optimised. This means \method only needs to run $k$-means to convergence once. \citet{ishioka2000extended} uses a very similar method to X-means, keeping a stack of clusters during training, and sequentially running $k$-means with $k$=2 on each. Again, this is much less efficient than \method, which does not need to run multiple models to convergence.

Another common usage of \ac{bic} is to select the number of clusters at a model selection stage. It is common in clustering applications to deal with the problem of unknown $k$ by training many $k$-means models with varying values of $k$, and selecting that with the lowest \ac{bic} \cite{zhang2014functional, lancaster2019reconceptualizing, salmanpour2022longitudinal}. Our experiments (Section \ref{subsec:sweep-k-comparison}) find that this is not only much slower than \method, as it requires running many models to convergence, but also less accurate, often severely overestimating $k$. A summary of the clustering algorithms discussed in this section and their parameters is presented in Table~\ref{tab:clustering_params}. 

\renewcommand{\arraystretch}{0.9} 
\begin{table}[t]
\caption{Common clustering algorithms and their required parameters}
\label{tab:clustering_params}
\centering
\resizebox{\textwidth}{!}{
\begin{tabular}{lp{9cm}}
\toprule
\textbf{Algorithm} & \textbf{Required Parameters} \\
\midrule
K-means & Number of clusters ($k$) \\
\addlinespace
%Hierarchical Clustering & Linkage method (single, complete, average, Ward); Distance metric; Optionally number of clusters for cut-off \\
%\addlinespace
Gaussian Mixture Models (GMM) & Number of components ($k$); Covariance type \\
\addlinespace
Spectral Clustering & Number of clusters ($k$); Affinity type \\
\addlinespace
Mean Shift & Bandwidth parameter (kernel width) \\
\addlinespace
Affinity Propagation & Preference parameter; Damping factor \\
\addlinespace
DBSCAN & Neighborhood radius (eps); Minimum points (minpts) \\
\addlinespace
HDBSCAN & Minimum cluster size; Minimum samples; Cluster selection eps\\
\addlinespace
OPTICS & Cut threshold for eps ($\xi$); Minimum neighborhood points (MinPts) \\
\addlinespace
X-Means & Maximum number of clusters; Minimum number of clusters\\
\addlinespace
\method & ---\\
\bottomrule
\end{tabular}
}
\end{table}

\section{The \method Algorithm} \label{sec:method}
\subsection{Quantifying Description Length} \label{subsec:quantifying-dl}
The Minimum Description Length (MDL) principle states that the best representation of the data is the one that can be specified exactly using the fewest number of bits.
%This is often treated as a model selection method, where multiple models are fit independently and then the one with smallest description length is chosen. For example, train an array of $k$-means models with varying $k$, then compare them all using \ac{mdl} and select the best. 
In \method, we quantify a bit cost for the various components of a clustering model and how they change over training. This allows \method to directly minimise the description length in a single procedure that simultaneously finds the optimal number of clusters, $k^*$, and fits a $k$-means model with $k^*$ clusters. The bitcost of a data point $x$ under a clustering model has two parts, the cost of specifying which cluster it belongs to, which we call the \emph{index cost}, and the cost of specifying its displacement from that cluster`s centroid, which we call the \emph{residual cost}. The former requires selecting an element of $\{0, \dots ,K-1\}$, thus taking $\log{K}$ bits. The latter can be approximated, by the Kraft-McMillan inequality, as $-\log{p(x|c)}$, where $c$ is the centroid of $x$'s assigned cluster. We model the cluster distribution as a multivariate normal distribution with unit variance
\begin{align*}
    p(x|c) &= \frac{1}{(2\pi)^{d/2}} \exp\left(-\frac{1}{2}(x-c)^T(x-c)\right) \\
    \iff -\log{p(x|c)} &= \frac{d \log{2\pi} + ||x-c||^2}{2} \,.
\end{align*}
The total cost of the data under the model is the sum of this cost for all data points, plus the cost of the model itself, which for $k$ clusters, $d$ dimensions and floating point precision $m$, is $kdm$ bits. (The precision $m$ is chosen from the data as the smallest value that allows perfect representation.) This is the quantity minimised by \method. Formally, let $X$ be the data to be clustered. Let $\Pi(X)$ be the set of all partitions of $X$, and let $|P|$ be the number of subsets in a partition. The optimal partition $P^*$ is
\begin{equation} \label{eq:mdl-objective}
   P^* = \argmin_{P \in \Pi(X)} |P|dm + |X|\log{|P|} + \frac{1}{2}\sum_{S \in P} Q(S)\,,
\end{equation}
where $Q$ computes the total sum of squares: $Q(X) = |X| \Var{X}$ and then $k^* = |P^*|$. (Full derivation is provided in Appendix \ref{app:mdl-objective-derivation}).

\begin{algorithm}[t]
\caption{K*-means Algorithm}
\label{alg:k-star-means}
\begin{algorithmic}[1]
\small
\Procedure{K*-means}{$X$}
    \State $\text{best\_cost} \gets \infty$
    \State $\text{unimproved\_count} \gets 0$
    \State $\mu \gets \frac{1}{n} \sum_{i=1}^n x_i$
    \State $C \gets [X]$
    \State $\mu_s \gets [\textproc{InitSubcentroids}(X)]$
    \State $C_s \gets \left[\left\{x \in X : \|x-\mu_{s_1}\| < \|x-\mu_{s_2}\|\right\}, \left\{x \in X : \|x-\mu_{s_2}\| < \|x-\mu_{s_1}\|\right\}\right]$
    
    \While{true}
        \State $\mu, C, \mu_s, C_s \gets \textproc{KmeansStep}(X, \mu, C, \mu_s, C_s)$
        \State $\mu, C, \mu_s, C_s, \text{did\_split} \gets \textproc{MaybeSplit}(X, \mu, C, \mu_s, C_s)$
        
        \If{$\neg \text{did\_split}$}
            \State $\mu, C, \mu_s, C_s \gets \textproc{KmeansStep}(X, \mu, C, \mu_s, C_s)$
            \State $\mu, C, \mu_s, C_s \gets \textproc{MaybeMerge}(X, \mu, C, \mu_s, C_s)$
        \EndIf
        
        \State $\text{cost} \gets \textproc{MDLCost}(X, \mu, C)$
        
        \If{$\text{cost} < \text{best\_cost}$}
            \State $\text{best\_cost} \gets \text{cost}$
            \State $\text{unimproved\_count} \gets 0$
        \Else
            \State $\text{unimproved\_count} \gets \text{unimproved\_count} + 1$
        \EndIf
        
        \If{$\text{unimproved\_count} = \text{patience}$}
            \State \textbf{break}
        \EndIf
    \EndWhile
    
    \State \Return $\mu, C$
\EndProcedure
\vspace{1em}
\Procedure{MDLCost}{$X, \mu, C$}
    \State $d \gets$ the dimensionality of $X$
    \State $floatprecision \gets -\log$ of the minimum distance between any values in $X$
    \State $floatcost \gets \frac{max(X)-min(X)}{floatprecision}$
    \State $modelcost \gets |C|d \times floatcost$
    \State $idxcost \gets |X|\log({|C|})$
    \State $c \gets$ the sum of the squared distances of every point in X from its assigned centroid
    \State $residualcost \gets \frac{|X|d\log({2\pi}) + c}{2}$
    \State \Return $modelcost + idxcost + residualcost$
\EndProcedure
\end{algorithmic}
\end{algorithm}

\subsection{Minimising Description Length} \label{subsec:minimising-dl}
In this section, we describe the algorithm by which \method efficiently optimises Equation~\eqref{eq:mdl-objective}. For a more formal exposition, see Algorithm~1. 
%We write all steps explicitly for clarity, though in practice some can be omitted for efficiency. For example, rather than recomputing the \ac{mdl} cost for each cluster in maybe-split, we can just compute the difference with the existing cost, which is faster. 
The familiar Lloyd's algorithm for $k$-means alternates between two steps: \texttt{assign}, which assigns each point to its nearest centroid, and \texttt{update}, which updates the centroids of each cluster to the mean of all of its assigned points. As well as the centroids and clusters, \method keeps track of subcentroids and subclusters. Subclusters consist of a partition of each cluster into two, and subcentroids are the means of all points in each subcluster. These are updated during the \texttt{update} and \texttt{assign} steps in just the same way as the main clusters and centroids. Essentially, each cluster has a mini version of $k$-means happening inside it during training. 

\method introduces two additional steps, \texttt{maybe-split} and \texttt{maybe-merge}, to the standard \texttt{assign-update} procedure. After the \texttt{assign} and \texttt{update} steps, the algorithm calls  \texttt{maybe-split}, which uses the subclusters and subcentroids to determine whether any cluster should be split. If no clusters are split, it proceeds with \texttt{maybe-merge}. In the case of a split, each constituent subcluster is promoted to a full cluster, and a new set of subclusters and subcentroids is initialised within each of them, following the k++-means initialisation method \cite{arthur2006k}. If two clusters are merged, their subclusters are discarded, and the clusters themselves are demoted to become two subclusters inside a new cluster that is their union. \method is initialised with just a single cluster containing all data points (and its two sub-clusters), and then cycles between \texttt{assign}, \texttt{update}, \texttt{maybe-split} and \texttt{maybe-merge} until the assignments remain unchanged for a full cycle. (In practice, for speed, we terminate if the cost has improved by $<2$ in the past 5 cycles. These are not core parameters of the algorithm, and can easily be omitted, in which the runtime is $\sim$30\% longer.) In this way, it simultaneously optimises~$k$ \emph{and} the standard $k$-means objective, with respect to Equation~\eqref{eq:mdl-objective}. 
%As the original $k$-means problem is NP-hard \cite{mahajan2012planar}, so is \method, and so we cannot of course give tight worst-case runtime guarantees. However, like $k$-means, this does not hurt its speed in practice (see Section \ref{subsec:runtime-analysis}).

\subsubsection{\texttt{Maybe-Split} Step}
This method (Algorithm 2) checks whether each cluster should be split into two. A naive approach would involve computing Equation~\eqref{eq:mdl-objective} for the current parameters and again with the given cluster replaced with its two subclusters, splitting if the latter is smaller. However, we can perform a faster, equivalent check by simply measuring the difference in cost.
If there are currently $k$ clusters, splitting would increase the index cost of each point by $\log(k+1) - \log(k) \approx 1/(k+1)$. It would also decrease the residual cost by $Q(S) - (Q(S_1) + Q(S_2))$, where $S$ is the original cluster and $S_1, S_2$ are its subclusters. To determine whether a split is beneficial, we compute $Q(S) - (Q(S_1) + Q(S_2))$ for every cluster. If any value exceeds $2N/(k+1)$, the cluster with the largest difference is split. 
%Thus, we compute $Q(S) - Q(S_1) + Q(S_2)$ for every cluster and if any is larger than $2N/(k+1)$, that with the largest is split. 

\subsubsection{\texttt{Maybe-Merge} Step}
This method (Algorithm 3) checks whether a pair of clusters should be merged. To avoid the time taken to compare every pair, we compare only the closest pair of centroids. Analogously to \texttt{maybe-split}, the potential change from merging is $\frac{1}{2}(Q(S) - (Q(S_1) + Q(S_2))) - N/k$, where  $S_1$ and $S_2$ are the two clusters with the closest centroids, and $S = S_1 \cup S_2$. If this value is positive, then $S_1$ and $S_2$ are merged, and become the new subclusters inside the new cluster $S$.

\begin{figure}[t]
\centering
\begin{minipage}{0.5\textwidth}
    \begin{algorithm}[H]
    \caption{\texttt{Maybe-Split} Procedure}
    \label{alg:maybe-split}
    \begin{algorithmic}[1]\small
    \Procedure{MaybeSplit}{$X, \mu, C, \mu_s, C_s$}
        \State $\text{best\_costchange} \gets \textproc{MDLCost}(X, \mu, C)$
        \State $\text{split\_at} \gets -1$
        
        \For{$i \in \{0, \ldots, |\mu|\}$}
            \State $subc1, subc2 \gets C_s[i]$
            \State $submu1, submu2 \gets \mu_s[i]$
            \State $costchange = + \sum_{x \in submu1} (x-subc1)^2 + \sum_{x \in submu2} (x-subc2)^2 - \sum_{x \in C[i]} (x-\mu[i])^2 + |X| / (|\mu| + 1)$
            %\State $\mu_{\text{split}}, C_{\text{split}} \gets \textproc{Split}(\mu, C, \mu_s, C_s, i)$
            %\State $\text{cost} \gets \textproc{MDLCost}(X, \mu_{\text{split}}, C_{\text{split}})$
            
            \If{$\text{costchange} < \text{best\_costchange}$}
                \State $\text{best\_costchange} \gets \text{costchange}$
                \State $\text{split\_at} \gets i$
            \EndIf
        \EndFor
        
        %\If{$\text{split\_at} \geq 0$}
        \If{$\text{best\_costchange} < 0$}
            \State $\mu,C~\gets~\textproc{Split}(\mu,C,\mu_s, C_s,\text{split\_at})$
            \State $\text{new}~\gets~\textproc{InitSubcentroids}(\mu_s[\text{split\_at}])$
            \State \mbox{$\mu_s\gets\mu_s[:\text{split\_at}]+\text{new}+\mu_s[\text{split\_at}:]$~}
        \EndIf
        
        \State \Return $\mu, C, \mu_s, C_s, \text{split\_at} \geq 0$
    \EndProcedure
    \end{algorithmic}
    \end{algorithm}
\end{minipage}\hfill
 \begin{minipage}{0.46\textwidth}
\small
   \begin{algorithm}[H]
    \caption{\texttt{Maybe-Merge} Procedure}
    \begin{algorithmic}[1]
    \small
   \Procedure{MaybeMerge}{$X,\mu,C,\mu_s,C_s$}
        %\State $\text{best\_cost} \gets \textproc{MDLCost}(X, \mu, C)$
        \State $i_1, i_2 \gets$ the indices of the closest pair of centroids
        \State $Z \gets C[i_1] \cup C[i_2]$
        \State $m_{\text{merged}} \gets \frac{1}{|Z|} \sum_{x \in Z} x$
        \State $mainQ \gets \sum_{z \in Z} (z-m_\text{merged})^2$
        \State $subcQ \gets \sum_{x \in C[i_1]} (x - \mu[i_1])^2 + \sum_{x \in C[i_2]}(x - \mu[i_2])^2$
        %\State $\text{cost}\gets\textproc{MDLCost}(X,\mu_{\text{merged}},C_{\text{merged}})$
        \State $costchange \gets mainQ - subcQ  - N/|\mu|$
        
        \If{$\text{costchange} < 0$}
            \State $C \gets C$ with $C[i_1]$ replaced with $Z$ and $C[i_2]$ removed
            \State $\mu \gets \mu$ with $\mu[i_1]$ replaced with $m_{\text{merged}}$ and $\mu[i_2]$ removed
        \EndIf
        
        \State \Return $\mu, C$
    \EndProcedure
    \end{algorithmic}
    \end{algorithm}
\end{minipage}
\end{figure}

\paragraph{Proof of Convergence} %\label{subsec:convergence-proof}
We prove in Appendix \ref{app:cnvrg-proof} that \method is guaranteed to converge in finite time. This is an extension of the proof of convergence for $k$-means, and uses the fact that all four of the steps at each cycle--\texttt{assign}, \texttt{update}, \texttt{maybe-split}, and \texttt{maybe-merge}--can only decrease the cost function in \eqref{eq:mdl-objective}.

\section{Experimental Evaluation} \label{sec:experimental-eval}
We evaluate our clustering algorithm with three sets of experiments. Firstly, we use synthetic data where we control the true number of clusters and test whether the algorithm can correctly identify this true number. Secondly, we measure performance on labelled data, and compare the predicted cluster labels to the true class labels using supervised clustering metrics. Thirdly, we examine the runtime as a function of dataset size, and show that it scales well compared to existing methods. 

\subsection{Synthetic Data} \label{subsec:synthetic-data}
For a range of values of~$k$, we first use Bridson sampling to sample~$k$ centroids in $\mathbb{R}^2$ near the origin with a minimum inter-point distance of $d$. Then we sample 1000/$k$ points 
%(round up for the first $1000 - k\lfloor\frac{1000}{k}\rfloor$, and round up otherwise) 
from a multivariate normal distribution, with unit variance, centred at each centroid. Examples of this synthetic data with varying $d$ are shown in Figure \ref{fig:synth-examples}. We then run \method, and comparison methods, on these 1,000 points and compare the number of clusters it finds to $k$. We repeat this 10 times each for each $(k,d) \in \{1, \dots, 50\} \times \{2,3,4,5\}$. For each value of $d$, there are 10 examples each of 50 different values of~$k$. We compute the accuracy, i.e fraction of these 500 examples with perfectly correct prediction of $k$, and also the mean squared error (MSE) from the predicted~$k$ to the true~$k$.

Table~\ref{tab:synth-results} presents the results. As can be seen, \method consistently outperforms the baseline algorithms in inferring the value of $k$. Unsurprisingly, its performance improves as the distance between centroids increases, and notably, the accuracy gap between \method and the baselines also widens under these conditions. \method reaches near perfect accuracy in the highly separable setting, c.f. the next highest of HDBSCAN at 58\%. Figure~\ref{fig:synth-examples} contains visualisations of the predictions of \method.

\begin{figure}[t]
    \centering
    \includegraphics[width=0.3\linewidth]{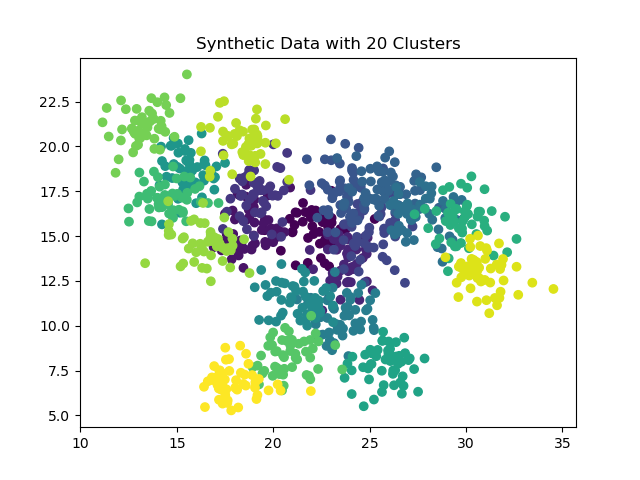}
    \includegraphics[width=0.3\linewidth]{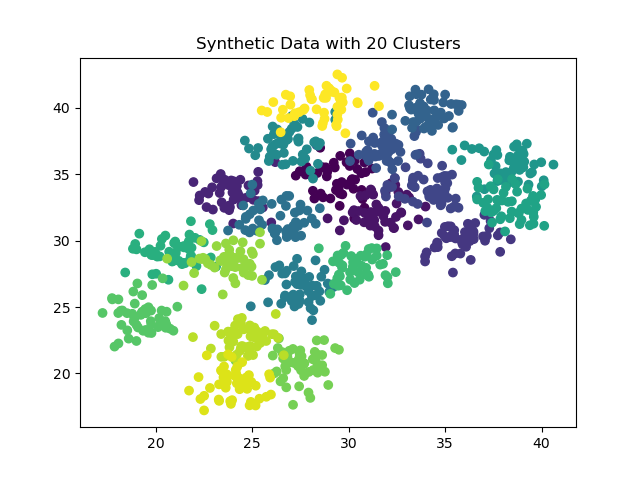}
    \includegraphics[width=0.3\linewidth]{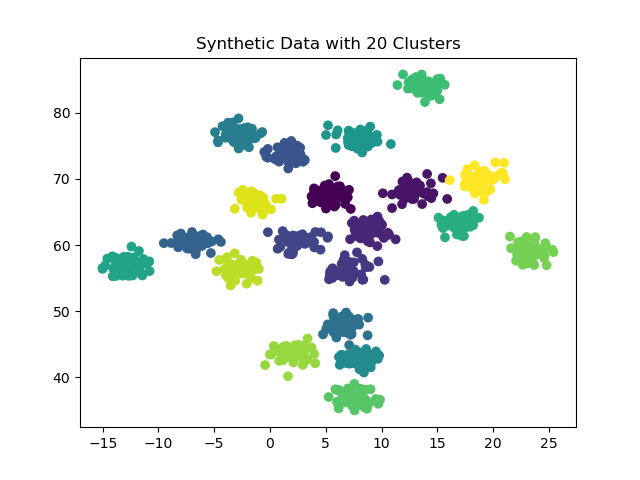}
    \caption{Synthetic data of standard, multivariate Normal clusters, with varying degrees of separation. Left: weak separation, inter-centroid distance constrained to~$\geq 2$, \method is 9\% accurate in inferring~$k$ and baselines are~$\leq$4.4\%. Middle: inter-centroid distance constrained to $\geq~3$, \method is~25\% accurate in inferring~$k$ and baselines are~$\leq$16\%. Right: strong separation, inter-centroid distance constrained to~$\geq 5$, \method is 99\%~accurate in inferring~$k$ and baselines are~$\leq$57\%.}
    \label{fig:synth-examples}
\end{figure}

\begin{table}[]
    \centering
    \caption{Performance on predicting the true number of clusters in synthetic data for varying degrees of cluster separation. \method performs consistently the best, and reaches near perfect accuracy when $d=5$.}
    \label{tab:synth-results}
\resizebox{\textwidth}{!}{
    \begin{tabular}{lrrrrrrrr}
\toprule
& \multicolumn{4}{c}{mse} & \multicolumn{4}{c}{acc} \\
\midrule
& \method & dbscan & hdbscan & xmeans & \method & dbscan & hdbscan & xmeans \\
\midrule
synthetic d=2 & 306.35 & 126.10 & 414.73 & 721.54 & 9.00 & 4.40 & 4.00 & 3.80 \\
synthetic d=3 & 81.70 & 252.41 & 116.35 & 681.97 & 25.40 & 5.40 & 7.80 & 16.00 \\
synthetic d=4 & 1.94 & 244.28 & 28.34 & 630.13 & 68.00 & 7.60 & 21.40 & 22.20 \\
synthetic d=5 & 0.00 & 238.18 & 12.83 & 623.99 & 99.80 & 6.60 & 57.60 & 25.40 \\
\bottomrule
\end{tabular}
}
\end{table}

\begin{table}[t]
 \caption{Accuracy on labelled datasets, using supervised metrics. Both with respect to predicting the number of clusters, $k$, and the accuracy with respect to the class labels, (ACC, NMI and ARI), \method significantly outperforms all existing models that do not know the true number of clusters. As an upper-bound, we include results from two standard clustering models that do know $k$, $k$-means and GMM. \method performs very close or even equivalent to these upper bounds, despite not having $k$  or any other parameter, specified. '-1' indicates the algorithm still had not converged after 10hrs.}
    \label{tab:main-results}
\resizebox{\textwidth}{!}{
\begin{tabular}{llllllll}
\toprule
 &  & ACC & ARI & NMI & $k$ & \makecell{Num \\ Outliers}\ & Runtime (s) \\
\midrule
& affinity & -1.00 & -1.00 & -1.00 & -1.00 & 0 & 36000+ \\
 & meanshift & 77.23 & 63.42 & 80.28 & 7.00 & 0 & 463.39 \\
 & DBSCAN & 68.75 & 54.84 & 77.66 & 6.00 & 1 & 2.95 \\
 & HDBSCAN & 79.73 & 84.84 & 70.70 & 1214.00 & 11190 & 46.24 (0.79) \\
 \multirow[t]{9}{*}{\shortstack{MNIST \\ domain = images \\ n classes = 10}} & OPTICS & 0.32 & 0.06 & 36.95 & 20696.00 & 11122 & 225.34 \\
 & xmeans & 40.89 (4.75) & 34.77 (4.70) & 64.74 (2.25) & 118.80 (23.59) & 0 & 16.40 (6.41) \\
 & \method & \textbf{91.26 (3.56)} & \textbf{84.99 (2.97)} & \textbf{87.44 (1.14)} & 10.90 (0.32) & 0 & 3.38 (0.39) \\
\hdashline
 & kmeans & 84.12 (8.13) & 79.64 (6.92) & 85.31 (2.73) & 10.00 & 0 & 0.09 (0.05) \\
 & GMM & 86.29 (7.05) & 82.61 (6.89) & 87.41 (2.74) & 10.00 & 0 & 0.78 (0.26) \\ \hline 
\multicolumn{8}{c}{~} \\ \toprule
& affinity & 61.29 & 49.37 & 75.94 & 25.00 & 0 & 148.15 \\
 & meanshift & 74.55 & 63.88 & 78.03 & 6.00 & 0 & 27.93 \\
 & DBSCAN & 80.46 & 71.00 & 83.51 & 7.00 & 0 & 0.13 \\
 & HDBSCAN & 77.49 & 82.16 & 79.09 & 108.00 & 829 & 0.80 (0.01) \\
 \multirow[t]{9}{*}{\shortstack{USPS \\ domain = images \\ n classes = 10}} & OPTICS & 1.88 & 0.46 & 43.95 & 2681.00 & 1617 & 11.68 \\
 & xmeans & 55.12 (5.03) & 46.09 (4.23) & 73.27 (1.71) & 41.00 (8.12) & 0 & 1.00 (0.43) \\
 & \method & \textbf{88.68 (0.00)} & \textbf{81.57 (0.00)} & \textbf{87.14 (0.00)} & 8.00 (0.00) & 0 & 0.80 (0.26) \\ 
 \hdashline
 & kmeans & 79.72 (8.15) & 78.68 (6.66) & 86.41 (2.12) & 10.00 & 0 & 0.03 (0.03) \\
 & GMM & 81.72 (6.76) & 80.27 (5.68) & 86.84 (1.82) & 10.00 & 0 & 0.11 (0.01) \\
\cline{1-8}
\multicolumn{8}{c}{~} \\ \toprule
& affinity & 41.49 & 27.73 & 57.58 & 46.00 & 0 & 233.19 \\
 & meanshift & 55.98 & 36.05 & 58.67 & 6.00 & 0 & 103.31 \\
 & DBSCAN & 26.09 & 3.70 & 22.00 & 3.00 & 1 & 0.22 \\
 & HDBSCAN & 51.62 & 46.01 & 55.52 & 402.00 & 4193 & 1.09 (0.02) \\
 \multirow[t]{9}{*}{\shortstack{Imagenet \\ (subset) \\ domain = images \\ n classes = 10}} & OPTICS & 1.35 & 0.28 & 40.22 & 3916.00 & 2262 & 18.11 \\
 & xmeans & 39.21 (3.19) & 25.53 (3.03) & 55.92 (0.88) & 70.00 (8.54) & 0 & 2.76 (0.68) \\
 & \method & \textit{66.18 (1.55)} & \textit{46.42 (1.45)} & \textit{60.20 (0.86)} & 6.40 (0.70) & 0 & 0.94 (0.34) \\
\hdashline
 & kmeans & \textbf{69.79 (5.18)} & \textbf{55.08 (4.65)} & \textbf{64.16 (2.81)} & 10.00 & 0 & 0.05 (0.04) \\
 & GMM & 66.85 (6.11) & 53.97 (5.44) & 64.01 (2.76) & 10.00 & 0 & 0.30 (0.09) \\ \hline
\multicolumn{8}{c}{~} \\ \toprule
& affinity & -1.00 & -1.00 & -1.00 & -1.00 & 0 & 36000+ \\
 & meanshift & 52.08 & 17.89 & 59.53 & 16.00 & 0 & 1205.21 \\
 & DBSCAN & 50.60 & 10.52 & 61.59 & 20.00 & 0 & 2.22 \\
 & HDBSCAN & 65.35 & 67.68 & 67.12 & 2453.00 & 24170 & 53.98 (7.86) \\
 \multirow[t]{9}{*}{\shortstack{Speech \\ Commands \\ domain = audio \\ n classes = 36}} & OPTICS & 0.74 & 0.13 & 47.76 & 28932.00 & 17163 & 325.78 \\
 & xmeans & 26.32 (7.78) & 14.33 (8.22) & 47.70 (18.56) & 190.10 (161.25) & 0 & 16.00 (13.01) \\
 & \method & \textit{68.73 (1.57)} & \textit{48.43 (2.49)} & \textit{70.22 (0.67)} & 26.50 (0.97) & 0 & 20.98 (5.22) \\
\hdashline
 & kmeans & \textbf{71.08 (1.72)} & \textbf{57.78 (1.67)} & 72.67 (0.47) & 36.00 & 0 & 0.30 (0.06) \\
 & GMM & 71.04 (1.27) & 56.12 (1.63) & \textbf{72.90 (0.42)} & 36.00 & 0 & 6.46 (0.89) \\
\cline{1-8}
\multicolumn{8}{c}{~} \\ \toprule
& affinity & 40.36 & 23.94 & \textit{48.27} & 75.00 & 0 & 597.33 \\
 & meanshift & 21.50 & 9.19 & 30.45 & 9.00 & 0 & 275.23 \\
 & DBSCAN & 16.40 & 1.98 & 18.59 & 12.00 & 0 & 0.40 \\
 & HDBSCAN & 30.08 & 24.05 & 47.72 & 664.00 & 6153 & 3.27 (0.03) \\
 \multirow[t]{9}{*}{\shortstack{20 NG \\ domain = text \\ n classes = 20}} & OPTICS & 1.60 & 0.27 & 43.73 & 5529.00 & 3211 & 28.84 \\
 & xmeans & 30.01 (10.66) & 15.48 (8.18) & 37.78 (19.83) & 107.60 (56.16) & 0 & 4.78 (2.51) \\
 & \method & \textit{42.33 (1.14)} & \textit{26.08 (0.44)} & 46.61 (0.67) & 11.20 (0.42) & 0 & 2.46 (0.96) \\
\hdashline
 & kmeans & 46.73 (1.47) & 33.68 (0.53) & 50.42 (0.48) & 20.00 & 0 & 0.07 (0.06) \\
 & GMM & \textbf{47.03 (1.22)} & \textbf{33.71 (0.78)} & \textbf{50.68 (0.50)} & 20.00 & 0 & 0.86 (0.19) \\
\cline{1-8}
\multicolumn{8}{c}{~} \\ \toprule
 & affinity & 36.60  & 18.12  & 40.75  & 46.00  & 0.00  & 31.09  \\
 & meanshift & 41.30  & 12.82  & 35.96  & 15.00  & 0.00  & 25.09  \\
 & DBSCAN & 37.65  & 11.51  & 39.23  & 27.00  & 0.00  & 0.05  \\
  & HDBSCAN & 18.40  & 5.90  & 45.39  & 321.00  & 1275.00  & 0.26 (0.73) \\
 \multirow[t]{9}{*}{\shortstack{MSRVTT \\ domain = video \& text \\ n classes = 20}}  & OPTICS & 4.30  & 0.70  & \textbf{46.86}  & 1763.00  & 1091.00  & 5.02  \\
 & xmeans & 31.78 (580.54) & 11.64 (965.93) & 24.36 (1441.93) & 32.40 (3924.62) & 0.00  & 0.39 (40.48) \\
 & \method & \textbf{44.10 (136.25)} & \textbf{25.75 (65.28)} & 38.16 (33.06) & 18.10 (87.56) & 0.00  & 2.57 (40.59) \\
\hdashline
 & kmeans & 40.07 (108.95) & 25.35 (128.11) & 38.43 (62.75) & 20.00  & 0.00  & 0.04 (1.01) \\
 & GMM & 41.41 (193.57) & 25.28 (101.87) & 38.44 (49.71) & 20.00  & 0.00  & 0.31 (9.16) \\

\bottomrule
\end{tabular}
}
\vspace{-2em}
   
\end{table}

\subsection{Labelled Datasets}
We evaluate \method on six datasets spanning multiple modalities. \textbf{MNIST} and \textbf{USPS} both consist of handwritten digit images from 0--9, \textbf{Imagenette} \cite{imagenette} is a subset of ImageNet with ten image classes, \textbf{Speech Commands} consists of short spoken words for command recognition in 36 classes, \textbf{20 NewsGroups} is a dataset of text documents across twenty topics, and \textbf{MSRVTT} consists of video clips paired with natural language captions in 20 categories. For all datasets, we dimensionally reduce with UMAP \cite{mcinnes2018umap} (min-dist=0, n-neighbours=10). For ImageNet we first apply CLIP \cite{pmlr-v139-radford21a} and for 20 Newsgroups we first take features from Llama-3.1 \cite{touvron2023llama} (mean across all tokens). For MSRVTT, we first take CLIP features of both the video and text (mean across all frames and tokens). As well as tracking the predicted number of classes, we assess partition quality by comparing to the ground truth partition arising from the class labels using three metrics: clustering accuracy (ACC), adjusted rand index (ARI), and normalised mutual information (NMI), as defined, for example in \cite{mahon2024hard}.

As baselines for clustering with unknown~$k$, we compare to the following:  affinity propagation \mbox{(damping factor = 0.5)}, mean shift (bandwidth = median of pairwise distances), DBSCAN (\texttt{eps=0.5}, \texttt{min-samples} = 5), HDBSCAN (\texttt{eps}=0.5, \texttt{min-samples} = 5), OPTICS, ($\xi=0.05$, \texttt{min-samples}=5), and XMeans (\texttt{kmax}=$\sqrt{\text{dataset size}}$). These methods are all described in Section \ref{sec:related-work} (see also Table~\ref{tab:clustering_params}). For XMeans, in the absence of any guidance on selecting \texttt{kmax}, we select it in this way because it is the value at which the information content is roughly equal between the index cost and the residual cost. All other parameter values are the sci-kit learn\footnote{\url{https://scikit-learn.org/stable/}} defaults. 

Our results are summarised in Table~\ref{tab:main-results}. \method consistently outperforms all other methods that do not require setting $k$. Meanshift and DBSCAN tend to underestimate $k$, while affinity propagation, HDBSCAN, XMeans, and OPTICS tend to overestimate it, often by a factor of~10 or more. \method, on average, slightly underestimates~$k$, but is much closer than existing methods. It is also much more accurate with respect to the clustering metrics, on some datasets (MNIST, USPS) even performing on par with $k$-means and GMM, which have the \emph{true} value of $k$~specified. 

Occasionally (20-NG, MSRVTT), one of the existing methods gets a high NMI score. However, we observe that they also vastly overpredict~$k$ in these cases, which means that there are very different numbers of classes in the true and predicted partitions. This can cause NMI to give unreliable results as the entropy in the latter is then unnaturally high. For existing methods, it is quite likely that one could obtain better results by manually tuning the parameters. We find that it is possible to get almost any value of $k$ by such tuning (see Appendix \ref{app:dbscan-vary-k}), but the focus of the present paper is on cases in which the user does not know the true value of~$k$. In other words, they do not have a ground truth against which to tune these parameters, and instead have to use the defaults. Table \ref{tab:main-results} shows that \method is a much better choice in such cases. Figure \ref{fig:cluster-visualisation} shows example visualisations of the predicted clusters for ImageNet and Speech Commands.

\begin{figure}[t]
    \centering
    \includegraphics[width=0.48\linewidth]{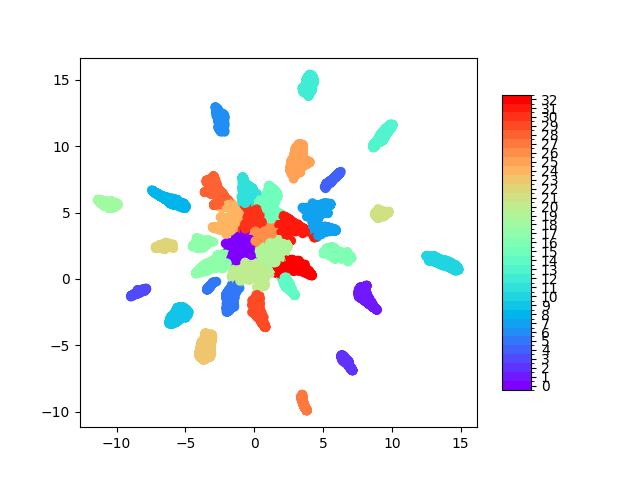}
    \includegraphics[width=0.48\linewidth]{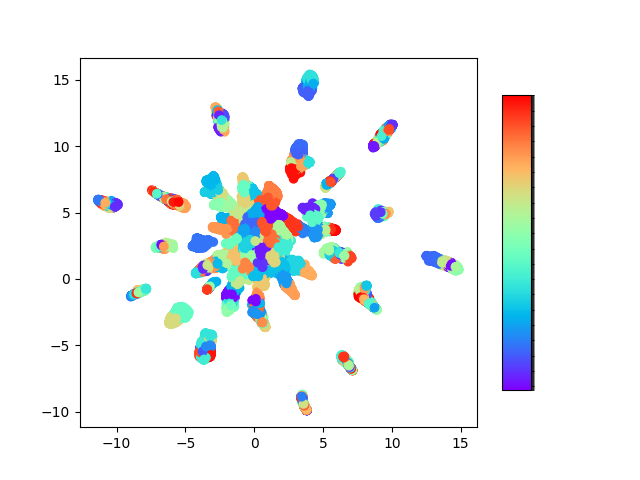}
    %\caption{Clusters predicted by \method for the UMAP representations, ImageNet (subset) (left) and Speech Commands (right). It predicts 9 classes for ImageNet (subset), vs 10 in the annotations, and 33 classes for Speech Commands, vs 36 in the annotations. }
    \caption{Clusters predicted by \method for the UMAP representations on the Speech Commands dataset, by \method (left) and XMeans (right). \method predicts 33 classes and XMeans predicts~315, vs.~36 in the annotations. }
    \label{fig:cluster-visualisation}
\end{figure}

\subsection{Comparison to Sweeping $k$} \label{subsec:sweep-k-comparison}
A common approach when clustering with unknown $k$ is to train $k$-means models with multiple values of $k$, compute some external model-selection criterion, commonly the Bayesian information criterion \cite{schwarz1978estimating} for each, and select whichever $k$ gives the lowest BIC \cite{wessman2012temperament,zhang2014functional, lancaster2019reconceptualizing, salmanpour2022longitudinal}.
Table \ref{tab:comparison-to-sweep-k} shows the performance of this common approach compared to \method. As we are simulating the scenario in which there is no knowledge of $k$, we sweep in increments of 10\% up to the dataset size. Sweeping plus BIC selection tends to favour very high values of $k$, generally 4-5x the number of annotated classes. It is also between 10 and 50x slower than \method. 

\begin{table}[t]
 \caption{Comparison of \method with the common approach of sweeping $k$ and selecting with BIC. \method is consistently faster and more accurate.} \label{tab:comparison-to-sweep-k}

\resizebox{\textwidth}{!}{
\begin{tabular}{llllllll}
\toprule
 &  & ACC & ARI & NMI & NC & Runtime (s) \\
\midrule
\multirow[t]{9}{*}{MNIST} & sweepkm & 12.86 & 8.17 & 56.01 & 25.00 & 148.15 \\
 & \method & \textbf{91.26 (3.56)} & \textbf{84.99 (2.97)} & \textbf{87.44 (1.14)} & 10.90 (0.32) & \textbf{3.38 (0.39)} \\
\cmidrule{1-7}
\multirow[t]{9}{*}{USPS}  & sweepkm & 32.36 (0.81) & 21.77 (0.77) & 65.20 (0.40) & 68.40 (3.10) & 11.21 (0.73) \\
& \method & \textbf{88.68 (0.00)} & \textbf{81.57 (0.00)} & \textbf{87.14 (0.00)} & 8.00 (0.00) & \textbf{0.80 (0.26)} \\ 
\cmidrule{1-7}
\multirow[t]{9}{*}{ImageNet (subset)} & sweepkm & 8.16 (0.26) & 1.18 (0.05) & 7.62 (0.05) & 83.20 (4.13) & 19.54 (0.23)\\
& \method & \textbf{66.18 (1.55)} & \textbf{46.42 (1.45)} & \textbf{60.20 (0.86)} & 6.40 (0.70) & \textbf{0.94 (0.34)} \\
\cmidrule{1-7}
\multirow[t]{9}{*}{Speech Commands} & sweepm & 32.19 (1.27) & 20.10 (0.90) & 62.29 (0.30) & 239.50 (12.12) & 951.58 (11.31) \\
 & \method & \textbf{68.73 (1.57)} & \textbf{48.43 (2.49)} & \textbf{70.22 (0.67)} & 26.50 (0.97) & \textbf{20.98 (5.22)} \\
\cmidrule{1-7}
\multirow[t]{9}{*}{20 NG} & sweepkm & 32.75 (0.54) & 17.44 (0.54) & 46.84 (0.17) & 107.30 (5.83) & 36.61 (1.28) \\
 & \method & \textbf{42.33 (1.14)} & \textbf{26.08 (0.44)} & 46.61 (0.67) & 11.20 (0.42) & \textbf{2.46 (0.96)} \\
\cmidrule{1-7}
\multirow[t]{9}{*}{MSRVTT} & sweepkm & 27.50 (89.64) & 12.24 (53.57) & \textbf{41.36 (18.64)} & 91.60 (464.76) & 7.33 (20.57) \\
 & \method & \textbf{44.10 (136.25)} & \textbf{25.75 (65.28)} & 38.16 (33.06) & 18.10 (87.56) & \textbf{2.57 (40.59)} \\
 \bottomrule
\end{tabular}
 }
 \end{table}

\subsection{Runtime Analysis} \label{subsec:runtime-analysis}
The runtimes from Table \ref{tab:main-results} already give an indication of the speed of \method compared with existing methods. To examine this further, and in particular how it depends on dataset size, we use subsets of varying size from the largest of the datasets from Table \ref{tab:main-results}: Speech Commands which has 99,000 data points. 
Figure~\ref{fig:runtime-by-nsamples} shows the runtime of \method, compared to the fastest baselines, on subsets of size $1,000, 2,000, \dots, 99,000$. 
The fastest at all sizes is $k$-means, which remains well under 1s even for 99,000 samples. 
The next is DBSCAN, rising to \ {} $\sim$3s, then the GMM{} $\sim$5s and \method \ {}$\sim$8s. 

\begin{figure}[t]
    \centering
    \includegraphics[width=.8\linewidth]{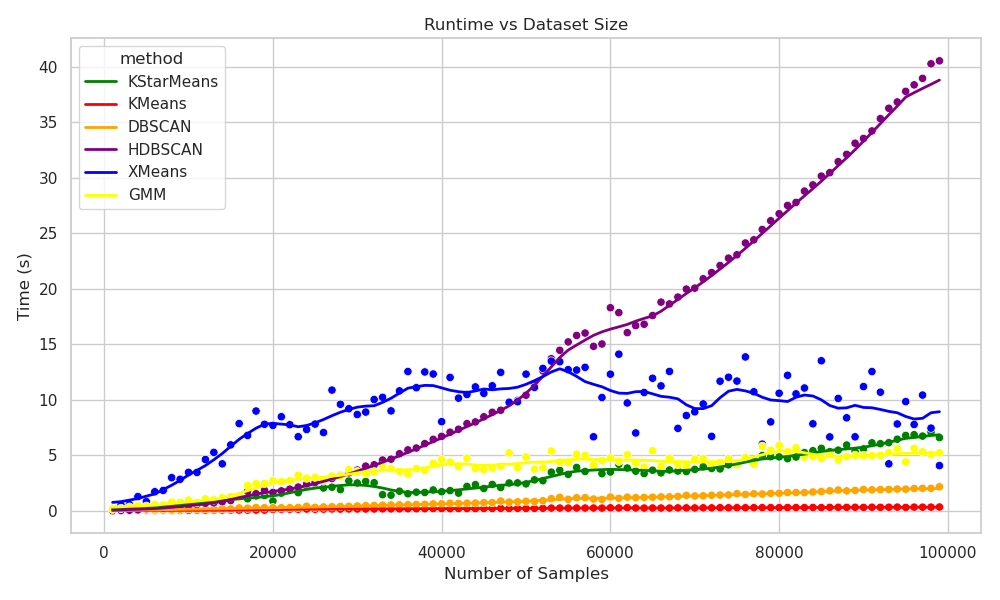}
    \caption{Runtime as a function of dataset size. Each point represents the mean runtime of 10 randomly sampled subsets from the Speech Commands dataset of the given size. The line tracks windowed averages. \method has a similar speed to xmeans, faster than HDBSCAN, especially for larger input sizes, and slower than kmeans, DBSCAN and GMM.}
    \label{fig:runtime-by-nsamples}
%    \vspace{-4em}
\end{figure}

HDBSCAN is efficient for small samples, faster than \method and GMM, but increases much faster, and by 99,000 samples its runtime is 6x that of \method. XMeans is the most erratic, by far the slowest for small sample sizes, but increasing very little or even decreasing, and ending up close to \method. The reason for the surprising decrease could be that XMeans predicts fewer clusters for larger datasets. It could also be related to the optimised C-Engine that the public XMeans code makes use of\footnote{\url{https://pyclustering.github.io}}. Note that Figure \ref{fig:runtime-by-nsamples} shows only the fastest five algorithms. Mean-shift, affinity propagation and OPTICS are all substantially slower and would be off the chart if included. 

\vspace{-0.5em}
\section{Limitations} \label{sec:limitations}
\vspace{-0.5em}
\method lacks a theoretical guarantee of worst-case runtime. Also, it assumes roughly spherical clusters, which often does not hold in practice. However, both of these criticisms also apply to the original $k$-means \cite{mahajan2012planar}, and they have not prevented it being the most used clustering algorithm.
%is reasonably fast, the original $k$-means is still faster, and so probably more useful when $k$ is known. Also \method does not currently designate outliers, like HDBSCAN, which may be useful in some applications. Future work can explore exluding outliers to minimise DL.

\vspace{-0.5em}
\section{Conclusion} \label{sec:conclusion}
\vspace{-0.5em}
This paper presented a new clustering algorithm, \method, which can be applied in the absence of knowing $k$ and does not require setting any other parameters such as thresholds. We prove that \method is guaranteed to converge, and we show empirically on synthetic data that it can more accurately infer $k$ than comparison methods, and with near-perfect accuracy for sufficiently separated centroids. We then test it on six labelled datasets spanning image, text, audio and video domains, and show that it is significantly more accurate than existing methods in terms of standard clustering metrics. We also compare it to the standard practice of sweeping~$k$ in $k$-means and selecting with a model selection criterion. Finally, we demonstrate how its runtime scales with dataset size, and show that it is faster, and scales better than the majority of existing methods. \method can be useful in cases where the user has large uncertainty as to the appropriate value of~$k$. 

\bibliography{bibliography}
\bibliographystyle{plainnat}

\appendix

\newpage 
\section{Dependence of $k$ on DBSCAN Parameters} \label{app:dbscan-vary-k}
Although DBSCAN does not explicitly require setting~$k$, its two key parameters,\texttt{eps} and \texttt{min-pts}, essentially determine a value for~$k$ indirectly. As can be seen in Figure~\ref{fig:enter-label}   the different values for $k$~found by DBSCAN for different values of eps and min-pts range from 6~to over~4,000. In general, smaller \texttt{eps} and smaller \texttt{min-pts} produces more clusters. The number of annotated classes is~10.

\begin{figure}[t]
    \centering
    \includegraphics[width=\linewidth]{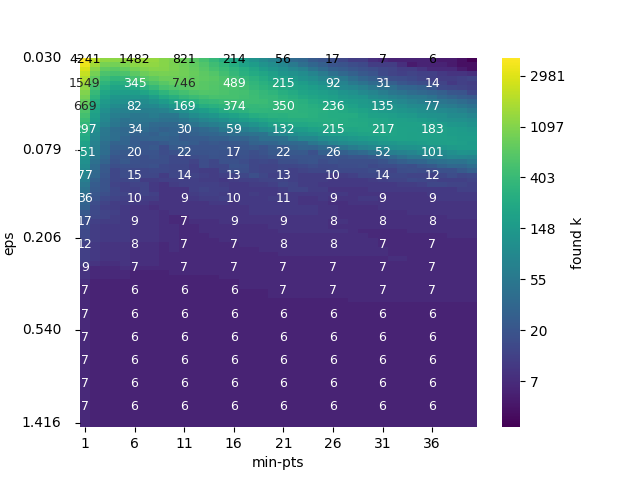}
    \caption{Values of $k$ (number of clusters) found on MNIST for different values of the DBSCAN parameters, \texttt{min-pts} (x-axis) and \texttt{eps} (y-axis). We sweep \texttt{min-pts} from 1--40, and \texttt{eps} from~0.03 to~1.5 in 5\%~increments.}
    \label{fig:enter-label}
\end{figure}

\FloatBarrier

\section{Derivation of MDL Clustering Objective} \label{app:mdl-objective-derivation}
The objective to derive is that from Section \ref{subsec:quantifying-dl}:
\begin{equation*}
   P^* = \argmin_{P \in \Pi(X)} |P|dm + |X|\log{|P|} + \frac{1}{2}\sum_{S \in P} Q(S)\,.
\end{equation*}
The first two terms are direct expressions of the cost to specify the centroids (each costs $dm$ bits and there are $|P|$ of them), and the cluster labels (each costs $\log{|P|}$ bits and there are $|X|$ of them. The third term arises from the expression for the negative log-probability, and the fact that we can drop additive constants in the argmin. Let $c(P,x)$ be the centroid of the cluster $x$ belongs to under partition $P$. Then
\begin{gather*}
    \argmin_{P \in \Pi(X)} |P|dm + |X|\log{|P|} + \sum_{x \in X} \frac{d \log{2\pi} + ||x-c(P,x)||^2}{2} = \\
    \argmin_{P \in \Pi(X)} |P|dm + |X|\log{|P|} + \frac{1}{2}\sum_{x \in X} ||x-c(P,x)||^2 = \\
    \argmin_{P \in \Pi(X)} |P|dm + |X|\log{|P|} + \frac{1}{2}\sum_{S \in P} \sum_{x \in S} ||x-c(P,x)||^2 = \\
    \argmin_{P \in \Pi(X)} |P|dm + |X|\log{|P|} + \frac{1}{2}\sum_{S \in P} Q(S)\,.
\end{gather*}

\section{Proof of Convergence} \label{app:cnvrg-proof}
The \method algorithm is guaranteed to converge. We now provide a proof of this fact.

\begin{lemma} \label{lemma:assign-decreases}
    At each \texttt{assign} step (which is the same step as vanilla $k$-means), the \ac{mdl} cost either decreases or remains the same, and it remains the same only if no points are reassigned.
\end{lemma}
\begin{proof}
As defined in Section \ref{subsec:quantifying-dl}, there are two components to the \ac{mdl} cost: the index cost, consisting of the bits to specify the cluster membership of each point, and the residual cost, consisting of the bits corresponding to the displacement of each from its cluster centre. The former depends only on the number of points $N$ and the number of clusters $k$, and is unaffected by re-assignment. The latter is proportional to the sum of squared distances of each point to its assigned centre. By definition of reassignment, if a point is reassigned, then it is closer to its new centroid than its old centroid. Thus, every reassignment does not affect the first two summands of the cost and strictly reduces the third.  
\end{proof}

\begin{lemma} \label{lemma:update-decreases}
    At each \texttt{update} step (which is the same step as vanilla $k$-means), the \ac{mdl} cost either decreases or remains the same, and it remains the same only if no centroids are updated.
\end{lemma}
\begin{proof}
As with the assign step, the update step does not change $k$ or $N$ and so does not affect the index cost. The latter can be written as the sum across clusters of the sum of all points in that cluster from the centroid. When the centroids are updated, they are updated to the mean of all points of points in the cluster, which is the unique minimiser of the sum of squared distances. As well as a standard statistical fact, this can be seen by observing that the $SS(x) = \sum_{i=1}^m (x-x_i)^2$ is a u-shaped function of $x$, so achieves its global minimum when 
\begin{gather*}
    SS'(x) = 0 \iff \\
    2 \sum_{i=1}^m x - x_i = 2mx - 2\sum_{i=1}^m x_i = 0 \iff \\
    x = \frac{1}{m}\sum_{i=1}^m x_i\,.
\end{gather*}
This holds for the reassignment of each cluster, and so for the whole reassignment step.
\end{proof}

\begin{theorem}
    The \method algorithm is guaranteed to converge in finite time.
\end{theorem}
\begin{proof}
    By Lemmas \ref{lemma:assign-decreases} and \ref{lemma:update-decreases}, the \ac{mdl}\ cost strictly decreases at each step at which points are reassigned and centroids updated. The other two steps, \texttt{maybe-split} and \texttt{maybe-merge}, include explicit steps that the \ac{mdl} cost decreases before being performed, so also are guaranteed to strictly decrease the \ac{mdl} cost. Together, this means the algorithm will never revisit an assignment during training. Moreover, there are a finite number of assignments, equal to the number of partitions of $N$ data points, which is given by the $N+1$th Bell number \cite{graham1994concrete}, $B_{N+1}$. Therefore \method cannot run for more than a finite number, namely $B_{N+1}$, steps.
\end{proof}

\begin{remark}
    This is an extension of the standard proof of convergence of $k$-means. Like for $k$-means, this proof establishes a theoretical worst-case run time which is exponential but then, in practice, the algorithm converges quickly. Practical empirical runtimes are studied in detail in Section \ref{sec:experimental-eval}.
\end{remark}

\end{document}